\documentclass[runningheads]{llncs}

\usepackage[utf8]{inputenc}
\usepackage{microtype}
\usepackage{amsmath, amssymb}
\usepackage[nodollardollar]{onlyamsmath}
\usepackage{mathtools}
\usepackage{wrapfig}
\usepackage{placeins}
\usepackage{tikz}
\usepackage{xcolor}
\usepackage{hyperref} 
\usepackage{caption}
\usepackage{nicefrac}
\usepackage{subcaption}
\usepackage[linesnumbered, algoruled, vlined]{algorithm2e}
\usepackage{pgfplots}
\usepackage{makecell}
\usepackage{fullpage}

\usetikzlibrary{arrows}
\tikzset{auto, >=stealth}
\tikzset{every edge/.append style={shorten >=1pt}}
\pgfplotsset{width=7cm,compat=1.15}
\pgfplotstableset{col sep=comma}
\usepgfplotslibrary{statistics}

\newcommand{\TracesStyle}[1]{\ensuremath{ \mathit{#1} }}
\newcommand{\Traces}{\ensuremath{ \TracesStyle{S} }}

\newcommand{\traceStyle}[1]{\ensuremath{ \mathit{#1} }}
\newcommand{\trace}{\ensuremath{ \traceStyle{u} }}

\newcommand{\formula}{\ensuremath{ \varphi }}

\newcommand{\Var}{\ensuremath{\mathit{Var}}}
\newcommand{\propVariables}{\ensuremath{\mathcal{P}}}
\newcommand{\operators}{\ensuremath{\Lambda}}

\newcommand{\ltrue}{\mathit{true}}
\DeclareMathOperator{\lfalse}{\mathit{false}}

\DeclareMathOperator{\limplies}{\rightarrow}

\DeclareMathOperator{\lnext}{\mathbf{X}}
\DeclareMathOperator{\luntil}{\mathbf{U}}
\DeclareMathOperator{\leventually}{\mathbf{F}}
\DeclareMathOperator{\lglobally}{\mathbf{G}}
\newcommand{\ltlset}{\ensuremath{\mathcal{F}}}
\newcommand{\LTL}{LTL}
\newcommand{\LTLf}{\LTL{}\textsubscript{f}}

\newcommand{\DT}{\ensuremath{t}}
\newcommand{\score}{\ensuremath{\mathit{s}_l}}
\newcommand{\scorer}{\ensuremath{\mathit{s}_{\mathit{r}}}}
\newcommand{\stopCriteria}{\ensuremath{\mathit{stop}}}
\newcommand{\sfunc}{\ensuremath{\mathit{s}}}
\newcommand{\leaf}{\ensuremath{\mathit{leaf}}}

\newcommand{\loss}{\ensuremath{\mathit{l}}}
\newcommand{\wloss}{\ensuremath{\mathit{wl}}}

\newcommand{\thres}{\ensuremath{\kappa}}
\newcommand{\wt}{\ensuremath{\mathit{w}}}
\newcommand{\wttrace}{\ensuremath{\Omega}}

\newcommand{\valFuncPos}[3]{\ensuremath{V({#1},{#2},{#3})}}
\newcommand{\valFunc}[2]{\ensuremath{V({#1},{#2})}}
\newcommand{\class}{\ensuremath{b}}

\newcommand{\x}[2]{\ensuremath{x_{#1,#2}}}
\newcommand{\lt}[2]{\ensuremath{l_{#1,#2}}}
\newcommand{\rt}[2]{\ensuremath{r_{#1,#2}}}
\newcommand{\y}[3]{\ensuremath{y_{#1,#2}^{#3}}}
\newcommand{\dagConst}[1]{\ensuremath{\Phi_{#1}^{\text{str}}}}
\newcommand{\satisfactionConst}[1]{\ensuremath{\Phi_{#1}^{stf}}}
\newcommand{\semanticConst}[2]{\ensuremath{\Phi^{#1}_{#2}}}
\newcommand{\specDepth}{\ensuremath{n}}
\newcommand{\propFormulaOrig}[1]{\Phi^{\Traces}_{#1}}
\newcommand{\minscore}{\mathit{min\_score}}
\newcommand{\wtratio}{\ensuremath{\Omega_r}}

\SetKwInput{Hyperparameter}{Hyperparameter}
\SetKwInput{Input}{Input}
\SetKwInput{Output}{Output}
\SetKwBlock{Loop}{loop}{end}

\SetKwInOut{Parameter}{Parameter\hspace{-0.2cm}}
\SetKwFunction{AlgoLearnLTL}{\text{Learn-LTL}}
\newcommand{\DTree}{\text{Decision-tree}}
\SetKwFunction{StopCriteria}{\text{Stopping-Criteria}}
\SetKwFunction{AlgoDTree}{\DTree}
\newcommand{\SplitSample}{\text{Split-sample}}
\SetKwFunction{AlgoSplitSample}{\SplitSample}
\newcommand{\Optimize}{\text{Infer-formula}}
\SetKwFunction{AlgoOptimize}{\Optimize}
\newcommand{\OptimizeMaxedSize}{\text{Infer-formula-max}}
\SetKwFunction{AlgoOptimizeMaxedSize}{\OptimizeMaxedSize}
\newcommand{\OptimizeSufficientScore}{\text{Infer-formula-min}}
\SetKwFunction{AlgoOptimizeSufficientScore}{\OptimizeSufficientScore}

\newcommand{\abs}[1]{\ensuremath{|#1|}}

\newcommand{\real}{\ensuremath{\mathbb{R}}}

\newcommand{\nat}{\ensuremath{\mathbb{N}}}

\title{Learning Linear Temporal Properties from Noisy Data:\texorpdfstring{\\}{} A MaxSAT Approach}
\titlerunning{}

\author{
	Jean-Raphaël Gaglione\inst{1}
	\and
	Daniel Neider\inst{2}
	\and
	Rajarshi Roy\inst{2}
	\and
	Ufuk Topcu\inst{3}
	\and
	Zhe Xu\inst{4}
}

\institute{
	Ecole Polytechnique, Palaiseau, France \and
	Max Planck Institute for Software Systems, Kaiserslautern, Germany \and
	University of Texas at Austin, TX, USA \and
	Arizona State University, AZ, USA
}

\begin{document}

\maketitle
\setcounter{footnote}{0}

\begin{abstract}
We address the problem of inferring descriptions of system behavior using Linear Temporal Logic (LTL) from a finite set of positive and negative examples. 
Most of the existing approaches for solving such a task rely on predefined templates for guiding the structure of the inferred formula.
The approaches that can infer arbitrary \LTL{} formulas, on the other hand, are not robust to noise in the data.
To alleviate such limitations, we devise two algorithms for inferring concise \LTL{} formulas even in the presence of noise.
Our first algorithm infers minimal \LTL{} formulas by reducing the inference problem to a problem in maximum satisfiability and then using off-the-shelf MaxSAT solvers to find a solution.
To the best of our knowledge, we are the first to incorporate the usage of MaxSAT solvers for inferring formulas in \LTL{}.
Our second learning algorithm relies on the first algorithm to derive a decision tree over \LTL{} formulas based on a decision tree learning algorithm. 
We have implemented both our algorithms and verified that our algorithms are efficient in extracting concise \LTL{} descriptions even in the presence of noise. 

\end{abstract}


\section{Introduction}
\label{sec:intro}

Explaining the behavior of complex systems in a form that is interpretable to humans has become a central problem in Artificial Intelligence. 
Applications where having concise system descriptions are essential include debugging, reverse engineering, motion planning, specification mining for formal verification, to name just a few examples.

For inferring descriptions of a system, we rely on a set of positive examples and a set of negative examples generated from the underlying system.
Given such data, the objective is to infer a concise model in a suitable formalism that is consistent with the data; that is, the model must satisfy the positive examples and not satisfy the negative ones.

Most of the data representing AI systems consist of sequences since, more often than not, the properties of these systems evolve over time.
For representing data consisting of sequences, temporal logic has emerged to be a successful and popular formalism.
Among temporal logics, Linear Temporal Logic (LTL), developed by Pneuli~\cite{pnueli}, enjoys being both mathematically rigorous and human interpretable for describing system properties.
Moreover, \LTL{} displays a resemblance to natural language and simultaneously eliminates the ambiguities existing in natural language.
To this end, \LTL{} uses modal operators such as $\leventually$ (``eventually''), $\lglobally$ (``globally''), $\luntil$ (``until''), and several others to describe naturally occurring sequences based on their temporal aspect.
One can use these operators to easily describe properties such as ``the robot should reach the goal and not touch a wall or step into the water in the process'' using $(\neg\mathtt{water} \land \neg \mathtt{wall} ) \luntil \mathtt{goal}$
or ``every request should be followed by a grant eventually'' using $\lglobally(\mathtt{request} \rightarrow(\leventually \mathtt{grant}))$.

The task of inferring temporal logic formulas consistent with a given data has been studied extensively~\cite{Kong2017,Bombara2016,zheletter,zhe_ADHS}.
Most of the existing inference methods, however, typically impose syntactic restrictions on the inferred formula.
In particular, these methods only derive formulas whose structures are based on certain handcrafted templates, which has several drawbacks.
First, handcrafting templates by users may not be a straightforward task since, it requires adequate knowledge about the underlying system.
Second, by restricting the structure of inferred formulas, we potentially increase the size of the inferred formula.

Nevertheless, there are approaches~\cite{DBLP:conf/fmcad/NeiderG18,DBLP:conf/aips/CamachoM19} that avoid the use of templates.
These works present algorithms that rely on reducing the learning problem to a Boolean satisfiability problem (SAT) to infer \LTL{} formulas that perfectly classify the input data.
However, such exact algorithms suffer from the limitation that they are susceptible to failure in the presence of noise which is ubiquitous in real-world data.
Furthermore, trying to infer formulas that perfectly classify a noisy sample often results in complex formulas, hampering interpretability.


In this paper, to alleviate the limitation of the earlier approaches, in this paper we present two novel algorithms for inferring \LTL{} formulas from data provided as a sample consisting of system traces labeled as positive and negative.
We use a variant of \LTL{} that is interpreted over finite traces and is commonly referred to as \LTLf{}, and is of particular interest for several applications related to AI~\cite{DBLP:journals/ai/BacchusK00}.
Now the goal of our algorithms is to infer concise \LTLf{} formulas that achieve a low \emph{loss} on the sample,
where loss $\loss(\Traces, \formula)$ refers to the fraction of examples in the sample $\Traces$ the inferred formula $\formula$ misclassified.
Precisely, the problem we solve is the following: given a sample $\Traces$ and a threshold $\thres$, find a minimal \LTLf{} formula $\formula$ that is consistent with $\Traces$ and has $\loss(\Traces, \formula)\leq\thres$.
Our algorithms are built upon the SAT-based learning algorithms introduced by Neider and Gavran~\cite{DBLP:conf/fmcad/NeiderG18}.
Our first algorithm tackles this problem by reducing the search of an \LTLf{} formula to a problem in maximum satisfiability.
Roughly speaking, we construct formulas in Boolean propositional logic with appropriate weights assigned to its various clauses.
We then search for assignments to the propositional formula that maximize the total weight of the satisfied clauses.
We then show that from an assignment that maximizes the weights of the satisfied clauses we can extract an \LTLf{} formula minimizing loss in a straightforward manner.

Our first algorithm constructs series of monolithic propositional formulas to model the inference problem and is, thus, often inefficient for inferring larger formulas.
Our second algorithm solves the inference problem by dividing the problem into smaller subproblems based on a decision tree learning algorithm.
Instead of finding \LTLf{} formulas that achieve a loss less than $\thres$ in one step, for each decision node in the tree we exploit our first algorithm to infer small \LTLf{} formulas.
Neider and Gavran also propose a similar decision tree based learning algorithm for LTL.
However, our algorithm outperforms theirs in two aspects. 
First, our algorithm is robust to noise in the data. Second, we incorporate a systematic search of \LTLf{} formulas for each decision node, while theirs rely on simple heuristics for searching without termination guarantees.

We have implemented a prototype of both of our algorithms,
and compared them to the algorithms by Neider and Gavran.
To effectuate the evaluation, we used benchmarks that model typical \LTL{} patterns used in practice.
From our observations, we conclude that our algorithms outperform that of Neider and Gavran in terms of running time and formula size, especially in the benchmarks consisting of noise.

\paragraph{Related Work} 
Our approach builds upon that of Neider and Gavran~\cite{DBLP:conf/fmcad/NeiderG18} who exploit a SAT-based inference method.
Similar to their work is the work of Camacho et al~\cite{DBLP:conf/aips/CamachoM19} which uses a SAT-based approach to construct Alternating Finite Automaton consistent with data and extract an \LTLf{} formula from it.
Most of the other works require templates for inferring LTL formulas. 
Among those, one prominent work is that of Kim et al~\cite{DBLP:conf/ijcai/KimMSAS19} as they infer satisfactory \LTLf{} formulas from noisy data, using the Bayesian inference problem.

For the inference of temporal logic formulas, certain works also exploit decision tree learning algorithms. 
One example is the work of Bombara et al~\cite{Bombara2016} which infers Signal Temporal Logic (STL) classifiers based on decision trees.
While their work can infer STL formulas with arbitrary misclassification error on the data, the STL primitives used for the decision nodes in their trees are derived only from a predefined set.
Another work is that of Brunello et al~\cite{DBLP:conf/jelia/BrunelloSS19} which infers decision trees over Interval Temporal Logic. The decision nodes in their trees, as well, are simple formulas; usually consisting of a single temporal relation with a proposition.

The inference problem of temporal logic, in general, has gained popularity in the recent years.
Apart from LTL, this problem has been looked at for a variety of logics, including Past Time Linear Temporal Logic (PLTL)~\cite{DBLP:conf/fmcad/ArifLERCT20}, Signal Temporal Logic (STL)~\cite{zheletter,Jin2015,zhe2016,Kong2017,Asarin2021,Hoxha2017}, Property Specification Language (PSL)~\cite{DBLP:conf/ijcai/0002FN20} and several others~\cite{zheCDC2019GTL,zhe2019ACCinfo,zhe2019privacy}.


\section{Preliminaries}\label{sec:exact-learning}

In this section, we introduce the necessary background required for the paper.
 
\paragraph{Propositional Logic.} Let $\Var$ be a set of propositional variables, which take Boolean values $\{0,1\}$ ($0$ represents $\ltrue$, $1$ represents $\lfalse$). Formulas in propositional logic---denoted by capital Greek letters---are defined inductively as follows:
\[ \Phi \coloneqq x\in\Var \mid \neg \Phi \mid \Phi \lor \Phi \]

Moreover, we add syntactic sugar and allow the formulas $\ltrue$, $\lfalse$, $\Phi\land\Psi$, $\Phi\rightarrow\Psi$ and $\Phi\leftrightarrow\Psi$ which are defined in the standard manner. 

A propositional valuation is a mapping $v\colon \Var \mapsto \{0,1\}$, which maps propositional variables to Boolean values. 
We define the semantics of propositional logic using a valuation function $\valFunc{v}{\Phi}$ that is inductively defined as follows: $\valFunc{v}{x}=v(x)$, $\valFunc{v}{\neg\Psi}= 1-\valFunc{v}{\Psi}$, and $\valFunc{v}{\Psi\lor\Phi}= \mathit{max}\{\valFunc{v}{\Psi} ,\valFunc{v}{\Phi}\}$. 
We say that $v$ satisfies $\Phi$ if  $\valFunc{v}{\Phi}=1$, and call $v$ as a model of $\Phi$. 
A propositional formula $\Phi$ is satisfiable if there exists a model $v$ of $\Phi$.

The satisfiability problem of propositional formula---abbreviated as SAT---is the problem of determining whether a propositional formula is satisfiable or not.
For the SAT problem, usually propositional formulas are assumed to be provided in Conjunctive Normal Form (CNF). 
Formulas in CNF are represented as conjunction of clauses $C_i$, where each clause is a disjunction of literals; a literal being a propositional variable $x$ or its complement $\neg x$.

\paragraph{Finite Traces.} Formally, a \emph{trace} over a set $\propVariables$ of propositional variables (which represent interesting system properties) is a finite sequence of symbols $\trace = a_0a_1\ldots a_n$, where $a_i\in 2^\propVariables$ for $i\in\{0,\cdots,n\}$.
For instance, $\{p,q\}\{p\}\{q\}$ is a trace over the propositional variables $\propVariables=\{p,q\}$. 
The empty trace, denoted by $\epsilon$, is an empty sequence.
The length of a trace is given by $\abs{u}$ (note $\abs{\epsilon}=0$).
Moreover, given a trace $u$ and  $i\in\mathbb{N}$, we use $u[i]$ to denote the symbol at position $i$ (counting starts from $0$). 
Finally, we denote the set of all traces by $(2^{\propVariables})^\ast$.

\paragraph{Linear Temporal Logic.} Linear temporal logic (LTL) is a logic 
that enables reasoning about sequences of events
by extending propositional Boolean logic with temporal modalities. Given a finite set $\propVariables$ of propositional variables, formulas in \LTL{}---denoted by small greek letters---are defined inductively by:
\[ \formula \coloneqq p\in\propVariables \mid \lnot\formula \mid \formula \lor \formula \mid \lnext\formula \mid \formula\luntil\formula \]

As syntactic sugar, we allow the use of additional constants and operator used in propositional logic. 
Additionally, we include temporal operators $\leventually$ (``future'') and $\lglobally$ (``globally'') by 
$\leventually \formula \coloneqq \ltrue \luntil \formula$ and $\lglobally\formula \coloneqq \neg \leventually \neg \formula$.
The set of all operators is defined as $\operators = \{ \neg, \lor, \land, \rightarrow, \lnext, \luntil, \leventually, \lglobally \}~\cup~\propVariables$ 
(propositional variables are considered to be nullary operators). 
We denote the set of all valid \LTL{} formulas as $\ltlset$.
We define the size $\abs{\formula}$ of an LTL formula $\formula$ to be the number of its unique subformulas.
For instance, size of formula $\formula=(p \luntil \lnext q) \lor \lnext q$ is 5, since, the distinct subformulas of $\formula$ are 
$p, q, \lnext q, p \luntil \lnext q$ and $(p \luntil \lnext q) \lor \lnext q$.

We use the semantics of \LTL{} over finite traces, introduced by Giacomo and Vardi~\cite{DBLP:conf/ijcai/GiacomoV13}.
For defining the semantics, we use a valuation function $V$, that maps a formula, a finite trace and a position in the trace to a boolean value.
Formally we define $V$ as follows: $\valFuncPos{p}{\trace}{i} = 1 \text{ if and only if } p\in \trace[i]$, $\valFuncPos{\neg\varphi}{\trace}{i} =  1- \valFuncPos{\varphi}{\trace}{i}$, $\valFuncPos{\varphi\lor \psi}{\trace}{i} = \mathit{max}\{\valFuncPos{\varphi}{\trace}{i}, \valFuncPos{\psi}{\trace}{i}\}$
, $\valFuncPos{\lnext\varphi}{\trace}{i} =  \mathit{min}\{i<\abs{u},\valFuncPos{\varphi}{\trace}{i+1}\}$, $\valFuncPos{\varphi\luntil\psi}{\trace}{i} = \mathit{max}_{i\leq j\leq \abs{u} }\{\valFuncPos{\varphi}{\trace}{j}, \mathit{min}_{i\leq k< j }\valFuncPos{\psi}{\trace}{k}\}$.

We say that a trace $\trace\in(2^\propVariables)^\ast$ satisfies a formula $\formula$ if $\valFuncPos{\formula}{\trace}{0}=1$. 
For the sake of brevity, we use $\valFunc{\trace}{\formula}$ to denote $\valFuncPos{\formula}{\trace}{0}$.

\section{Problem Formulation}

The input data is provided as a sample $\Traces \subset (2^{\propVariables})^\ast \times \{0,1\}$ consisting of labeled traces.
Precisely, sample $\Traces$ is a set of pairs $(\trace, \class)$, where $\trace\in(2^{\propVariables})^\ast$ is a trace and $\class\in\{0,1\}$ is its classification label.
The traces labeled $1$ are called positive traces, while the ones labeled $0$ are called the negative traces. 
We assume that in a sample $(\trace, \class_1)=(\trace, \class_2)$ implies $\class_1=\class_2$, indicating that no trace can be both positive and negative.
Further, we denote the size of $\Traces$, that is, the number of traces in a sample, by $\abs{\Traces}$.

We define a \emph{loss} function which assigns a real value to a given sample $\Traces$ and an \LTLf{} formula $\formula$.
Intuitively, a loss function evaluates how ``well'' the \LTLf{} formula $\formula$ classifies a sample.
While there are numerous ways loss functions can be defined (e.g., quadratic loss function, regret, etc.), we use the definition:
\[\loss(\Traces,\formula) = \sum_{(\trace,\class)\in\Traces}\frac{\abs{\valFunc{\formula}{\trace}-\class}}{\abs{\Traces}},\label{eq:loss} \]which calculates the fraction of traces in $\Traces$ which the \LTLf{} formula $\formula$ misclassified.

Having defined the setting, we now formally describe the problem we solve:
\begin{problem}\label{prob:ltl-learning}
Given a sample $\Traces \subset (2^{\propVariables})^\ast \times \{0,1\}$ and threshold $\thres\in[0,1]$, find an \LTLf{} formula $\formula$ such that $\loss({\Traces},{\formula})\leq\thres$.
\end{problem}

Generally speaking, the above problem is trivial without any constraints on the size of the inferred formula. 
The reason behind is that one can always find a large \LTLf{} formula that achieves a loss of zero.
We can construct such a formula $\formula$ in the following manner:
construct formulas $\formula_{u, v}$, for all $(u,1)\in\Traces$ and $(v,0)\in\Traces$, such that $\valFunc{\formula_{u, v}}{u}=1$ and $\valFunc{\formula_{u, v}}{u}=0$, using a sequence of $\lnext$-operators and an appropriate propositional formula to describe the first symbol where $u$ and $v$ differ; now $\varphi=\bigvee_{(u, 1) \in \Traces} \bigwedge_{(v, 0) \in \Traces} \varphi_{u, v}$ is the desired formula.
The formula $\varphi$, however, is large in size (of the order of $\abs{\Traces}^2\times \mathit{max}_{(\trace,\class)\in\Traces}\abs{u}$) and it does not help towards the goal of inferring a concise description of the data.

Our first algorithm for solving Problem~\ref{prob:ltl-learning}, in fact, infers an \LTLf{} formula that is minimal among the ones that achieve $\loss(\Traces, \formula)\leq\thres$.
We describe the algorithm in Section~\ref{sec:max-sat-learning}.
Our second algorithm infers a decision tree over \LTLf{} formulas, which is not guaranteed to be of minimal size.
However, decision trees are considered to be structures that provide human understandable explanations of the underlying system.
Further, in this algorithm described in Section~\ref{sec:DT-learner}, we allow a tunable parameter that makes it possible to adjust the size of the decision tree.

\section{Learning \LTLf{} formulas minimizing loss}
\label{sec:max-sat-learning}

Our solution to Problem~\ref{prob:ltl-learning} relies on MaxSAT solvers which we introduce next. 

\subsection{MaxSAT}
MaxSAT---a variant of the SAT (Boolean satisfiability) problem---is the problem of finding an assignment that maximizes the number of satisfied clauses in a given propositional formula provided in CNF.
For solving our problem, we use a more general variant of MaxSAT, known as Partial Weighted MaxSAT.
In this variant, a weight function $\wt \colon \mathcal{C} \mapsto \real\cup\{\infty\}$ assigns a weight to every clause in the set of clauses $\mathcal{C}$ ofs a propositional formula.  
The problem is to then find a valuation $v$ that maximises $\Sigma_{C_i\in\mathcal{C}} \wt(C_i) \cdot \valFunc{v}{C_i}$.

While the MaxSAT problem and its variants can be solved using dedicated solvers, standard SMT solvers like Z3~\cite{DBLP:conf/tacas/MouraB08} are also able to handle such problems.
According to terminology derived from the theory behind such solvers, clauses $C_i$ for which $\wt(C_i) = \infty$ are termed as \emph{hard} constraints, while, clauses $C_i$ for which $\wt(C_i) < \infty$ are termed as \emph{soft} constraints.
Given a propositional formula with weights assigned to clauses, MaxSAT solvers try to find a valuation that satisfies all the hard constraints and maximizes the total weight of the soft constraints that are satisfied.

\subsection{The learning algorithm} 
By using MaxSAT solvers that possess the capability of handling Partial Weighted MaxSAT problems, we can solve a stronger version of Problem~\ref{prob:ltl-learning}.
In this version, we assume that the loss based on which we search for \LTLf{} formulas has the following form:
\[\wloss(\Traces, \formula, \wttrace) = \sum\limits_{(\trace, \class)\in\Traces} \wttrace(\trace)|\valFunc{\formula}{\trace} - \class|, \label{eq:weighted-loss}\] 
where $\wttrace$ is a function that assigns a positive real-valued weight to each $\trace$ in the sample in such a way that $\sum_{(\trace, \class)\in\Traces}\wttrace(\trace) = 1$.
Observe that by considering the function $\wttrace(\trace)=\nicefrac{1}{\abs{\Traces}}$ for all traces in the sample, we have exactly $\wloss(\Traces, \formula, \wttrace)=\loss(\Traces, \formula)$ which is used in Problem~\ref{prob:ltl-learning}.
In addition to having a solution to Problem~\ref{prob:ltl-learning}, solving the stronger version provides us with a versatile algorithm that we exploit for learning decision trees over \LTLf{} formulas in Section~\ref{sec:DT-learner}.

For solving this problem, we devise an algorithm based on ideas from the learning algorithm of Neider and Gavran for inferring \LTLf{} formulas that perfectly classify a sample.
Following their algorithm, we translate the problem of inferring \LTLf{} formulas into problems in Partial Weighted MaxSAT and then use an optimized MaxSAT solver to find a solution.
More precisely, we construct a propositional formula $\propFormulaOrig{\specDepth}$ and assign weights to its clauses in such a way that an assignment $v$ of $\propFormulaOrig{\specDepth}$ that satisfies all the hard constraints and maximizes the weight of the soft constraints, satisfies two properties:
\begin{enumerate} 
\item $\propFormulaOrig{\specDepth}$ contains sufficient information to extract an \LTLf{} formula $\formula_v$ of size $\specDepth$; and
\item the sum of weights of the soft constraints satisfied by it is equal to $1-\wloss(\Traces, \formula_{v}, \wttrace)$.
\end{enumerate}

\begin{algorithm}[t]
	\small
	\KwIn{A sample $\Traces$, $\wttrace$ function, Threshold $\thres$}
	\DontPrintSemicolon
	$n\gets 0$\;
	{
		\Repeat{Sum of weights of soft constraints $\geq 1-\thres$}
		{	$n\gets n+1$\;
			Construct formula $\propFormulaOrig{\specDepth}=\dagConst{\specDepth}\land\satisfactionConst{\specDepth}$\;
			Assign weights to soft constraints in $\propFormulaOrig{\specDepth}$:\\ \qquad$\wt(\y{\specDepth}{0}{\trace})=\wttrace(\trace)$ for$(\trace,1)\in\Traces$, and $\wt(\neg\y{\specDepth}{0}{\trace})=\wttrace(\trace)$ for $(\trace,0)\in\Traces$\;
			Find assignment $v$ using MaxSAT solver}
		}
		\Return $\formula_v$\;	
	\caption{MaxSAT-based learning algorithm for \LTLf{}
	} 
	\label{alg:max-sat-learner}
\end{algorithm}

To obtain a complete algorithm, we increase the value of $n$ (starting from 1) until we find an assignment $v$ of $\propFormulaOrig{\specDepth}$ that satisfies the hard constraints and ensures that sum of weights of the soft constraints is greater than $1-\thres$.
The termination of this algorithm is guaranteed by the existence of a trivial \LTLf{} formula with achieves a zero loss on the sample (discussed earlier in this section).

On a technical level, the formula $\propFormulaOrig{\specDepth}$ in Algorithm~\ref{alg:max-sat-learner} is the conjunction $\propFormulaOrig{\specDepth}=\dagConst{\specDepth}\land\satisfactionConst{\specDepth}$, where $\dagConst{\specDepth}$ encodes the \emph{structure} of the prospective \LTLf{} formula (of size $\specDepth$) and $\satisfactionConst{\specDepth}$ tracks the satisfaction of the prospective \LTLf{} formula with words in $\Traces$.
We now explain each of the conjuncts in further detail.

\paragraph{Structural constraints.} For designing the formula $\dagConst{\specDepth}$, we rely on a canonical syntactic representation of \LTLf{} formulas, which we refer to as \emph{syntax DAGs}.
A syntax DAG is essentially a syntax tree (i.e., the unique tree that arises from the inductive definition of an \LTLf{} formula) in which common subformulas are shared.
As a result, the number of the unique subformulas of an \LTLf{} formula coincides with the number of nodes which we term as the size of its syntax DAG.

In a syntax DAG, to uniquely identify the nodes, we assign identifiers $1,\ldots,\specDepth$ in such a way that the root node is always indicated by $\specDepth$ and every node has an identifier larger than that of its children, if it has any. An example of a syntax DAG is shown in Figure~\ref{fig:syntax-dag}).

\begin{figure}
	\centering
	\begin{tikzpicture}[-,auto, scale=1]
	\node (1) at (0, 0) {$\lor$};
	\node (2) at (-.5, -.7) {$\luntil$};
	\node (3) at (.5, -.7) {$\leventually$};
	\node (4) at (-1.0, -1.4) {$p$};
	\node (5) at (0, -1.4) {$\lglobally$};
	\node (6) at (0, -2.1) {$q$};
	\draw[->] (1) -- (2); 
	\draw[->] (1) -- (3);
	\draw[->](2) -- (4);
	\draw[->] (2) -- (5);
	\draw[->] (3) -- (5);
	\draw[->] (5) -- (6);
	\end{tikzpicture}
	\hspace{0.7cm}
	\begin{tikzpicture}
	\node (1) at (0, 0) {$6$};
	\node (2) at (-.5, -.7) {$4$};
	\node (3) at (.5, -.7) {$5$};
	\node (4) at (-1.0, -1.4) {$1$};
	\node (5) at (0, -1.4) {$3$};
	\node (6) at (0, -2.1) {$2$};
	\draw[->] (1) -- (2); 
	\draw[->] (1) -- (3);
	\draw[->](2) -- (4);
	\draw[->] (2) -- (5);
	\draw[->] (3) -- (5);
	\draw[->] (5) -- (6);
	\end{tikzpicture}
	\caption{Syntax DAG and identifiers of the formula $(p \luntil \lglobally q) \lor \leventually \lglobally q$}
	\label{fig:syntax-dag}
\end{figure}

To encode the structure of a syntax DAG using propositional logic, we introduce the following propositional variables:
$\x{i}{\lambda}$ for $i\in\{1,\cdots,\specDepth\}$ and $\lambda\in\operators$, which encode that Node~$i$ is labeled by operator $\lambda$ (includes propositional variables); and $\lt{i}{j}$ and $\rt{i}{j'}$, for $i\in\{2,\cdots,\specDepth\}$ and $j,j'\in\{1,\cdots,i-1\}$, which encode that the left and right child of Node~$i$ is Node~$j$ and Node~$j'$, respectively. 
For instance, we must have variables $\x{6}{\land}$, $\lt{6}{4}$, and $\rt{6}{5}$ to be true in order to obtain a syntax DAG where Node~$6$ is labeled with $\land$, has the left child to be Node~$4$, and the right child to be Node~$5$ (similar to the syntax DAG in Figure~\ref{fig:syntax-dag}).  

We now introduce the following constraints on the variables to ensure that they encode a valid syntax DAG:
\begin{align}
\Big[ \bigwedge_{1\leq i \leq n} \bigvee_{\lambda \in \Lambda} \x{i}{\lambda} \Big] &\land \Big[ \bigwedge_{1 \leq i \leq n} \bigwedge_{\lambda \neq \lambda' \in \Lambda} \lnot \x{i} {\lambda} \lor \lnot \x{i}{\lambda'} \Big]\label{eq:label-uniqueness}\\    
[\bigwedge\limits_{2\leq i\leq  n} \bigvee\limits_{1\leq j\leq i}\lt{i}{j}]&\wedge[\bigwedge\limits_{2\leq i\leq  n}\bigwedge\limits_{1\leq j\leq j'\leq n}\neg \lt{i}{j}\vee \neg \lt{i}{j'}]\label{eq:left-child-uniqueness}\\
[\bigwedge\limits_{2\leq i\leq  n} \bigvee\limits_{1\leq j\leq i}\rt{i}{j}]&\wedge[\bigwedge\limits_{2\leq i\leq  n}\bigwedge\limits_{1\leq j\leq j'\leq n}\neg \rt{i}{j}\vee \neg \rt{i}{j'}]\label{eq:right-left-uniqueness}\\
\bigwedge_{\substack{2 \leq i \leq n, 1 \leq j, j'< i\\ \lambda\in\{\lnext, \luntil, \neg, \lor \}}} [ \x{i}{\lambda} \land \lt{i}{j} \land \rt{i}{j'} ] &\rightarrow \Big[ \bigvee_{\lambda' \in \Lambda} \x{j}{\lambda'} \land \bigvee_{\lambda'\in\Lambda} \x{j'} {\lambda'} \Big]\label{eq:ltl-ordering}\\
\bigvee\limits_{p\in\propVariables} \x{1}{p}\label{eq:first-node-prop}
\end{align}
Formula~\ref{eq:label-uniqueness} ensures that each node of the syntax DAG has a unique label.
Similarly, Formulas~\ref{eq:left-child-uniqueness} and~\ref{eq:right-left-uniqueness} ensure that each node of a syntax DAG has a unique left and right child, respectively. 
Finally, Formula~\ref{eq:first-node-prop} ensures that Node~$1$ is labeled by a propositional variable. 
Now, the formula $\dagConst{\specDepth}$ is obtained by taking conjunction of all the constraints discussed above.

Observe that from a valuation $v$ satisfying $\dagConst{\specDepth}$ one can extract an unique syntax DAG describing an \LTLf{} formula $\formula_v$ as follows: label Node~$i$ of the syntax DAG with the unique $\lambda$ for which $v(\x{i}{j})=1$;
assign Node~$\specDepth$ to be the root node; and assign edges from a node to its children based on the values of $\lt{i}{j}$ and $\rt{i}{j}$.

\paragraph{Semantic constraints.} Towards the definition of the formula $\satisfactionConst{\specDepth}$, we define propositional formulas  $\semanticConst{\specDepth}{\trace}$ for each trace $\trace$ that tracks the valuation of the \LTLf{} formula encoded by $\dagConst{\specDepth}$ on $\trace$.  
These formulas are built using variables $\y{i}{\tau}{\trace}$, where $i\in\{1,\ldots,\specDepth\}$ and $\tau\in\{1,\ldots,\abs{\trace}-1\}$, that corresponds to the value of $\valFuncPos{\formula_i}{\trace}{\tau}$ ($\formula_i$ is the \LTLf{} formula rooted at Node~$i$).
Now, to make sure that these variables have the desired meaning, we impose the following constraints based on the semantics of the \LTLf{} operators: 
\begin{align}
\bigwedge\limits_{1\leq i \leq n}\bigwedge\limits_{p\in \propVariables}\x{i}{p}&\rightarrow\Big[\bigwedge\limits_{0\leq \tau< \abs{u}}
\begin{cases}
\y{i}{\tau}{\trace}\text{ if }p\in u[i] \\
\neg \y{i}{\tau}{\trace}\text{ if }p\not\in u[i]
\end{cases}\Big]\label{eq:prop-semantics}\\
\bigwedge\limits_{\substack{1\leq i \leq n \\ 1\leq j< i}}\x{i}{\neg}\wedge \lt{i}{j} &\rightarrow\Big[\bigwedge\limits_{\substack{0\leq \tau< \abs{u}}}\Big[\y{i}{\tau}{\trace}\leftrightarrow \neg\y{i}{\tau}{\trace}\Big]\Big]\label{eq:negation-semantics}\\
\bigwedge\limits_{\substack{1\leq i \leq n \\ 1\leq j,j'< i}}\x{i}{\vee}\wedge \lt{i}{j}\wedge \rt{i}{j'}&\rightarrow\Big[\bigwedge\limits_{\substack{0\leq \tau < \abs{u}}}\Big[\y{i}{\tau}{\trace}\leftrightarrow \y{j}{\tau}{\trace}\vee \y{j'}{ \tau}{\trace}\Big]\Big]\label{eq:disjunction-semantics}\\
\bigwedge\limits_{\substack{1\leq i \leq n \\ 1\leq j< i}}\x{i}{\lnext}\wedge \lt{i}{j} &\rightarrow\Big[\bigwedge\limits_{\substack{0\leq \tau< \abs{u}-1}}\Big[\y{i}{\tau}{\trace}\leftrightarrow \y{i+1}{\tau}{\trace}\Big]\Big]\label{eq:next-semantics}\\
\bigwedge\limits_{\substack{1\leq i \leq n \\ 1\leq j, j'< i}}\x{i}{\luntil}\wedge \lt{i}{j}\wedge \rt{i}{j'}&\rightarrow\Big[\bigwedge\limits_{\substack{0\leq \tau< \abs{u}}}\Big[\y{i}{\tau}{\trace}\leftrightarrow \bigvee\limits_{\tau\leq \tau'< \abs{u}}\Big[\y{j'}{\tau'}{\trace}\wedge\bigwedge\limits_{\tau\leq t< \tau'} \y{j}{t}{\trace}\Big]\Big]\label{eq:until-semantics}
\end{align}
The constraints are similar to the ones proposed by Neider and Gavran, except that they have been adapted to comply with the semantics of \LTLf{}.
Formula~\ref{eq:prop-semantics} implements the semantics of propositions and states that if Node~$i$ is labeled with $p\in\propVariables$, then $\y{i}{\tau}{\trace}$ is set to 1 if and only if $p\in \trace[i]$.
Formulas~\ref{eq:negation-semantics} and~\ref{eq:disjunction-semantics} implement the semantics of negation and disjunction, respectively: if Node~$i$ is labeled with $\neg$ and Node $j$ is its left child, then $\y{i}{\tau}{\trace}$ equals the negation of $\y{j}{\tau}{\trace}$; on the other hand, if Node~$i$ is labeled with $\lor$, Node $j$ is its left child, and Node $j'$ is its right child, then $\y{i}{\tau}{\trace}$ equals the disjunction of $\y{j}{\tau}{\trace}$ and $\y{j'}{\tau}{\trace}$.
Formula~\ref{eq:next-semantics} implements the semantics of the $\lnext$-operator and states that if Node~$i$ is labeled with $\lnext$ and its left child is Node~$j$, then $\y{i}{\tau}{\trace}$ equals $\y{j}{\tau+1}{\trace}$.
Finally, Formula~\ref{eq:until-semantics} implements the semantics of the $\luntil$-operator; it states that if Node~$i$ is labeled with $\luntil$, its left child is Node~$j$, and its right child is Node~$j'$, then $\y{i}{\tau}{\trace}$ is set to 1 if and only if there exists a position $\tau'$ for which $\y{j'}{\tau'}{\trace}$ is set to 1 and for all positions $t$ lying between $\tau$ and $\tau'$, $\y{j}{t}{\trace}$ is set to 1.
The formula $\semanticConst{\specDepth}{\trace}$ is the conjunction of all of the semantic constraints described above.

We now define $\satisfactionConst{\specDepth}$ to be:

\begin{align}\label{eq:consistency-constraints}
\satisfactionConst{\specDepth}  = \bigwedge\limits_{(\trace, \class)\in\Traces} \semanticConst{\specDepth}{\trace} \land
\bigwedge\limits_{(\trace,1)\in\Traces} \y{\specDepth}{0}{\trace} \land
\bigwedge\limits_{(\trace,0)\in\Traces} \neg\y{\specDepth}{0}{\trace}
\end{align}

\paragraph{Weight assignment.} For assigning weights to the clauses of $\propFormulaOrig{\specDepth}$, we first convert the formulas $\dagConst{\specDepth}$ and $\satisfactionConst{\specDepth}$ into CNF.
Towards this, we simply exploit the Tseitin transformation~\cite{Tseitin1983} which converts a formula into an equivalent formula in CNF whose size is linear in the size of the original formula.

We now assign weights to constraints starting with the hard constraints as follows: $\wt{(\dagConst{\specDepth})} = \infty, \wt{(\semanticConst{\specDepth}{\trace})} = \infty \text{ for all } (\trace, \class)\in \Traces$.
Here, $\wt{(\Phi)} = w$ is a shorthand to denote $\wt(C_i)= w$ for all clauses $C_i$ in $\Phi$. 
The constraint $\dagConst{\specDepth}$ is a hard one since, it ensures that we obtain a valid syntax DAG of an \LTLf{} formula. 
$\semanticConst{\specDepth}{\trace}$ ensures that the prospective \LTLf{} formula is evaluated on the trace $\trace$ according to the semantics of \LTLf{} and thus, also needs to be a hard constraint.

The soft constraints are the ones that enforce correct classification and we assign them weights as follows: $\wt{(\y{\specDepth}{0}{\trace})} = \wttrace(\trace) \text{ for all } (\trace, 1)\in\Traces, \text{ and }
\wt(\neg\y{\specDepth}{0}{\trace}) = \wttrace(\trace) \text{ for all } (\trace, 0)\in\Traces$. Recall that $\wttrace$ refers to the function assigning weight to the traces.

To prove the correctness of our learning algorithm, we first ensure that the formula $\propFormulaOrig{\specDepth}$ along with the weight assigned to its clauses serve our purpose.
\begin{lemma}\label{lem:maxsat-encoding-correctness}
Let $\Traces$ be a sample, $\wttrace$ the weight function, $n\in\nat\setminus\{0\}$ and $\propFormulaOrig{\specDepth}$ the formula with the associated weights as defined above. Then,
\begin{enumerate}
	\item the hard constraints are satisfiable; and 
	\item if $v$ is an assignment that satisfies the hard constraints and maximizes the sum of weight of the satisfied soft constraints, then $\formula_v$ is an \LTLf{} formula of size $\specDepth$, such that $\wloss(\Traces, \formula_v, \wttrace)\leq\wloss(\Traces, \formula, \wttrace)$ for all \LTLf{} formulas $\formula$ of size $\specDepth$.
\end{enumerate}
\end{lemma}

\begin{proof}
	The hard constraints of $\propFormulaOrig{\specDepth}$ are $\dagConst{\specDepth}$ and $\semanticConst{\specDepth}{\trace}$. 
	Now, $\dagConst{\specDepth}$ is satisfiable since there always exists a valid \LTLf{} formula of size $\specDepth$.
	As a result, using the syntax DAG of a \LTLf{} formula of size $\specDepth$, we can find an assignment to the variables of $\dagConst{\specDepth}$ that makes it satisfiable. 
	The constraint $\semanticConst{\specDepth}{\trace}$, on the other hand, simply tracks the valuation of the prospective formula on traces $\trace$.
	One can easily find an assignment of the variables of $\semanticConst{\specDepth}{\trace}$ using the semantics of \LTLf{}.
	
	For proving the second part, let us assume that $v$ is an assignment that satisfies the hard constraints.
	We now claim that the sum of the weights of the satisfied soft constraints is equal to $1-\wloss(\Traces, \formula_v, \wttrace)$.
	If we can prove this, then if $v$ is an assignment that maximizes the weight of the satisfied soft constraints directly implies that $\formula_v$ minimizes the $\wloss$ function.
	Now towards proving the claim, we have the following:
	\begin{align*}
	\wloss({\Traces},{\formula_v}, \wttrace) &= \sum\limits_{\valFunc{\formula_v}{\trace}\neq \class} \wttrace(\trace) =  \sum\wttrace(\trace) - \sum\limits_{\valFunc{\formula_v}{\trace}= \class} \wttrace(\trace)\\
	&= 1 - \sum\limits_{\valFunc{\formula_v}{\trace}= \class} \wttrace(\trace) = 1 - \sum\limits_{v(\y{\specDepth}{0}{\trace})=\class} \wttrace(\trace)
	\end{align*}
	All the summations appearing in the above equation are over $(\trace, \class)\in\Traces$.
	Moreover, the quantity $\sum_{v(\y{\specDepth}{0}{\trace})=\class} \wttrace(\trace)$, appearing in the final line, refers to sum of the weights of the satisfied soft constraints, since
	the constraints in which $v(\y{\specDepth}{0}{\trace})=\class$ are the ones that are satisfied.
\end{proof}

The termination and the correctness of Algorithm~\ref{alg:max-sat-learner}, which is established using the following theorem, is a consequence of Lemma~\ref{lem:maxsat-encoding-correctness}.

\begin{theorem}\label{thm:approx-learning-correctness}
	Given a sample $\Traces$ and threshold $\thres\in\real$, Algorithm~\ref{alg:max-sat-learner} computes an \LTLf{} formula $\formula$ that has $\wloss(\Traces,\formula,\wttrace)\leq\thres$ and is the minimal in size among all \LTLf{} formulas that have $\wloss(\Traces, \formula, \wttrace)\leq\thres$.
\end{theorem}

\begin{proof}
	The termination of Algorithm~\ref{alg:max-sat-learner} is guaranteed by the fact that there always exists an \LTLf{} formula $\varphi$ for which $\wloss(\formula,\Traces, \wttrace)=0$ as discussed in the beginning of the section~\ref{sec:max-sat-learning}.
	Second, the fact that $\formula$ has $\wloss(\formula,\Traces, \wttrace)\leq\thres$ is a consequence of Lemma~\ref{lem:maxsat-encoding-correctness} 
	Finally, the minimality of the formula is consequence of the fact that Algorithm~\ref{alg:max-sat-learner} searches for \LTLf{} formula in increasing order of size.
\end{proof}

\section{Learning \LTLf{} formulas using decision trees} 
\label{sec:DT-learner}

In this section, we first introduce decision trees over \LTLf{} formulas and then proceed to discuss how we infer them from given data using Algorithm~\ref{alg:dt-learner}.
 
\subsection{Decision Trees over \LTLf{} formulas}

A decision tree over \LTLf{} formulas is a tree-like structure where all nodes of the tree are labeled by \LTLf{} formulas.
While the leaf nodes of a decision tree are labeled by either $\ltrue$ or $\lfalse$, the inner nodes are labeled by (non-trivial) \LTLf{} formulas which represent decisions to predict the class of a trace.
Each inner node leads to two subtrees connected by edges, where the left edge is represented with a solid edge and the right edge with a dashed one. Figure~\ref{fig:DT-example} depicts a decision tree over \LTLf{} formulas.

\begin{figure}
	\centering
	\begin{tikzpicture}
	\node (1) at (0, 0) {$\formula_1$};
	\node (2) at (-.8, -.7) {$\formula_2$};
	\node (3) at (.8, -.7) {$\ltrue$};
	\node (4) at (-1.4, -1.6) {$\ltrue$};
	\node (5) at (-.2, -1.6) {$\lfalse$};
	\draw[->]        (1) -- (2);
	\draw[->,dashed] (1) -- (3);
	\draw[->]        (2) -- (4);
	\draw[->,dashed] (2) -- (5);
	\end{tikzpicture}
	\caption{A decision tree over \LTLf{} formulas}
	\label{fig:DT-example}
\end{figure}
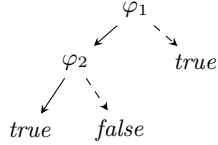

A decision tree $\DT$ over \LTLf{} formula corresponds to an \LTLf{} formula $\formula_t \coloneqq \bigvee_{\rho\in\Pi} \bigwedge_{\formula\in\rho} \formula^\prime$, where $\Pi$ is the set of paths that originate in the root node and end in a leaf node labeled by $\ltrue$ and $\formula^\prime = \formula$ if it appears before a solid edge in $\rho\in\Pi$, otherwise $\formula^\prime = \neg\formula$.

For evaluating a decision tree $\DT$ on a trace $\trace$, we use the valuation $\valFunc{\formula_\DT}{\trace}$ of the equivalent \LTLf{} formula $\formula$ on $\trace$.
We can, in fact, extend the valuation function and loss function for \LTLf{} formulas to decision trees as $\valFunc{\DT}{\trace}=\valFunc{\formula_\DT}{\trace}$ and $\loss(\Traces,\DT)=\loss(\Traces, \formula_\DT)$ respectively.

\subsection{The learning algorithm}

Our decision tree learning algorithm shares similarity with the class of decision tree learning algorithms known as Top-Down Induction of Decision Trees (TDIDT)~\cite{DBLP:journals/ml/Quinlan86}.
Popular decision tree learning algorithms such as ID3, C4.5, CART are all part of the TDIDT algorithm family.
In such algorithms, decision trees are constructed in a top-down fashion by finding suitable features (i.e., predicates over the attributes) of the data to partition it and then inductively applying the same method for the individual partitions.
Being a part of the TDIDT family, our algorithm can incorporate standard heuristics (e.g., tree pruning) to obtain a smaller derived decision tree as common in the other algorithms.

Algorithm~\ref{alg:dt-learner} outlines our approach to infer a decision tree over \LTLf{} formulas.
In our algorithm, we first check the stopping criteria (Line~\ref{line:stop-criteria}) that is responsible for the termination of the algorithm. 
If the chosen stopping criterion is met, we return a leaf node. 
We discuss the exact stopping criterion used in our algorithm in Section~\ref{subsec:stop-criteria}.

If the stopping criterion fails, we search for an ideal \LTLf{} formula $\formula$ using Algorithm~\ref{alg:max-sat-learner} for the current node of the decision tree.
Our search for $\formula$ is based on a score function and we infer the minimal one that achieves a score greater than a user-defined $\minscore$ on the sample.
The choice of the score function and parameter $\minscore$ is a crucial aspect of the algorithm, and further discussion about it is in Section~\ref{subsec:ltl-primitives}.

After having inferred formula $\formula$, next we split the sample into two sub-samples $\Traces_1$ and $\Traces_2$ with respect to $\formula$ as follows: $\Traces_1 = \{(\trace, \class)\mid \valFunc{\formula}{\trace}=1\}$, and $\Traces_2= \{(\trace, \class)\mid \valFunc{\formula}{\trace} = 0\}$.
Finally, we recursively apply the decision tree learning on each of the resulting sub-samples (Line~\ref{line:recursive-call}).
The decision tree returned is a tree with root node $\formula$ and subtrees $\DT_1$ and $\DT_2$.
\begin{algorithm}[th]
	\DontPrintSemicolon
	
	\Input{Sample $\Traces$, Minimum score value $\minscore$, Threshold $\thres$}
	\Parameter{ Stopping criteria $\stopCriteria$, Score function $\sfunc$}
	\BlankLine

	\uIf{$\stopCriteria(\Traces, \thres)$\label{line:stop-criteria}}
		{
			\Return{$\leaf(\Traces)$}
		}
	\Else{
		Infer minimal formula $\formula$ with $\sfunc(\Traces, \formula)\geq\minscore$ using Algorithm \ref{alg:max-sat-learner} \label{line:infer_form} \;
		Split $\Traces$ into $\Traces_1$, $\Traces_2$ using $\formula$ \label{line:split_sample}
		\;
		Infer trees $\DT_1$, $\DT_2$ by recursively appying algorithm to $\Traces_1$ and $\Traces_2$ \; \label{line:recursive-call}
		\Return{decision tree with root node $\formula$ and subtrees $\DT_{1}$, $\DT_{2}$} \;
		}
	\caption{Decision tree learning algorithm}
	\label{alg:dt-learner}
\end{algorithm}

\subsection{\LTLf{} formulas at each decision node}\label{subsec:ltl-primitives}
Ideally, we aim to infer \LTLf{} formulas at each decision node, that in addition to being small, also ensure that the resulting sub-samples after a split are as ``homogenous'' as possible. 
In simpler words, we would like the sub-samples obtained after a split to predominantly consist of traces of one particular class.
More homogenous splits result in early termination of the algorithm resulting in small decision trees.
To achieve this, one could simply infer a minimal \LTLf{} formula that perfectly classifies the sample.
While in principle, this solves our problem, in practice inferring an \LTLf{} formula that perfectly classifies a sample is a computationally expensive process~\cite{DBLP:conf/fmcad/NeiderG18}.
Moreover, it results in a trivial decision tree consisting of a single decision node.
Thus, to avoid that, we wish to infer concise \LTLf{} formulas that classify most traces correctly on the given sample.

To mechanize the search for concise \LTLf{} formulas for producing splits, we measure the quality of the \LTLf{} formula 
using a \emph{score} function.
In our algorithm, we infer a minimal LTL formula that achieves a score greater than a user-defined threshold $\minscore$.
This parameter regulates the tradeoff between the number of decision nodes in the tree and the size of the LTL formulas in each node.
While all TDIDT algorithms involve certain metrics (e.g., gini impurity, entropy) to measure the efficacy of a feature to perform a split,
these metrics are based on non-linear operations on the fraction of examples of each class in a sample.
However, searching \LTLf{} formulas based on such metrics cannot be handled using a MaxSAT framework.

One possible choice of score $\score(\Traces, \formula)=1-\loss(\Traces, \formula)$, which relies on the loss function (as in Definition~\ref{eq:loss}).
A formula $\formula$ with $\score(\Traces, \formula)\geq\minscore$ is a formula with $\loss(\Traces, \formula)\leq 1-\minscore$.
Thus, for inferring LTL formulas with score greater than $\minscore$, we invoke Algorithm~\ref{alg:max-sat-learner} to produce a minimal LTL formula $\formula$ with $\loss(\Traces,\formula)\leq 1-\minscore$.
Note that, for this score, one must choose the $\minscore$ to be smaller than $1-\thres$, else one would end up with a trivial decision tree with a single decision node.
Also, choosing $\minscore<0.5$, would always result in a leaf node.

While $\score$ as the metric seems to be an obvious choice, it often results in a problem which we refer to as \emph{empty splits}.
Precisely, the problem of empty splits occurs when one of the sub-samples, i.e., either $\Traces_1$ or $\Traces_2$ becomes empty.
Empty splits lead to an unbounded recursion branch of the learning algorithm, since, the best LTL formula chosen to $\score$ does not produce any meaningful splits.
This problem is more prominent in examples where the sample is skewed towards one class of examples, as has been often noticed in our experiments.
For instance, consider a sample $\Traces=\{(u,1)\}\cup \{(v_1,0), (v_2,0),\cdots (v_{99},0)\}$; for this sample if one searches for an \LTLf{} formula with $\minscore=0.9$, $\lfalse$ is a minimal formula; this formula, however, results in empty splits, since $\Traces_1=\emptyset$.

To address this problem, we use a score that relies on $\wloss$ with a weight function $\wtratio$ defined as follows:
\begin{align*}
\wtratio(\trace) = \frac{0.5}{\abs{\{(\trace, \class)| \class=1\}}}
\text{ for }(\trace, 1)\in \Traces, 
\wtratio(\trace) = \frac{0.5}{\abs{\{(\trace, \class)| \class=0\}}} 
\text{ for }(\trace, 0)\in \Traces
\end{align*}
Intuitively, the above $\wtratio$ function normalizes the weight provided to traces, based on the number of examples in its class, and reduces the imbalance in skewed samples.

Our final choice of score, based on the above $\wtratio$ function, is $\scorer(\Traces, \formula)=\mathit{max}\{\wloss(\Traces, \formula, \wtratio), 1-\wloss(\Traces, \formula, \wtratio)\}$ 
Using such a score, in addition to avoiding empty splits, we avoid always having \emph{asymmetric splits}.
We say a split is asymmetric when the fraction of positive examples in $\Traces_1$ 
is greater than or equal $0.5$ or the fraction of negative examples in $\Traces_2$ is less than or equal to $0.5$.
Choosing the score to be $1-\wloss(\Traces, \formula, \wtratio)$ always leads to asymmetric splits, since $\formula$ in order to minimize $\wloss(\Traces, \formula, \wtratio)$ attempts to satisfy several positive traces or not satisfy several negative examples.
In the decision tree learning algorithm, we are interested in homogenous splits and do not wish $\Traces_1$ (or $\Traces_2$) to predominantly have positive (or negative) traces and thus, the choice of $\scorer$. 

Now, for finding an appropriate LTL formula according to $\scorer$, we need to invoke Algorithm~\ref{alg:max-sat-learner} twice. 
Invoking Algorithm~\ref{alg:max-sat-learner} with sample $\Traces$, weight function $\wtratio$, and threshold $1-\minscore$, provides us with a formula $\formula_1$ that minimizes $\wloss(\Traces, \formula, \wtratio)$ and thus maximizes $1-\wloss(\Traces, \formula, \wtratio)$.
On the other hand, invoking Algorithm~\ref{alg:max-sat-learner} with an altered sample $\Traces_{R}$, in which labels of all traces are inverted (i.e. positives traces become negative and vice-versa), and keeping other arguments same, provides us with an formula $\formula_2$ that maximizes $\wloss(\Traces, \formula_2, \wtratio)$. 
Now, among $\formula_1$ and $\formula_2$, we choose the one that provides us with a higher score according to $\scorer$.

While any score function that avoids the problem of empty and asymmetric splits is sufficient for our learning algorithm, we have used $\scorer$ as a score function due to efficient performance using $\scorer$ in our experiments.  
Moreover, we show if we infer an \LTLf{} formula that achieves a $\scorer$ greater than $0.5$ in the algorithm, we never encounter empty splits using the following lemma.
\begin{lemma}\label{lem:scorearl-guarentee}
	Given a sample $\Traces$ and an \LTLf{} formula $\formula$, if $\scorer(\Traces, \formula)>0.5$, then there exists traces $\trace_1, \trace_2$ in $\Traces$ such that $\valFunc{\trace_1}{\formula}=1$ and $\valFunc{\trace_2}{\formula}=0$.
\end{lemma}

\begin{proof}
	Towards contradiction, without loss of generality, let us assume that for all $\trace$ in $\Traces$ and formula $\formula$ with $\scorer(\Traces, \formula)>0.5$, we have $\valFunc{\trace}{\formula}=1$.
	In such a case, $\abs{\valFunc{\trace}{\varphi}-\class}=0$ for $(\trace, 1)\in\Traces$ and $\abs{\valFunc{\trace}{\varphi}-\class}=1$ for $(\trace, 0)\in\Traces$.
	We can, thus, calculate that $\sum_{(u,1)\in\Traces}\abs{\valFunc{\trace}{\varphi}-\class}=0$,  $\sum_{(u,0)\in\Traces}\abs{\valFunc{\trace}{\varphi}-\class}=\abs{\{(\trace, 0)\in\Traces|\class=0\}}$, and consequently $\scorer(\Traces, \formula)=0.5$, violating our assumption.
\end{proof}
 
\subsection{Stopping Criteria}\label{subsec:stop-criteria}
The stopping criteria is essential for the termination of the algorithm.
Towards the definition of the stopping criteria, we define the following two quantities:
\[p_1(\Traces)=\frac{\abs{\{(\trace,\class)\mid \class=1\}}
}{\abs{\Traces}} \text{ and }p_2(\Traces)=\frac{\abs{\{(\trace,\class)\mid \class=0\}}
}{\abs{\Traces}}\]
The stopping criteria can now be defined as the follows:
\begin{align*}
\stopCriteria(\Traces)= \begin{cases}
\ltrue \text{ if } p_1(\Traces)\leq\thres \text{ or } p_2(\Traces)\leq\thres\\
\lfalse \text{ otherwise }
\end{cases}
\end{align*}
Intuitively, the above stopping criteria enforces that the algorithm terminates when the fraction of positive examples or the fraction of negative examples in a resulting sample is less or equal to $\thres$. 
Now, when the stopping criteria holds, the algorithm halts and returns a leaf node labeled by $\leaf(\Traces)$ where $\leaf$ is defined as $\leaf(\Traces)=\lfalse$ if $p_1\leq\thres$ and $\ltrue$ if $p_2\leq \thres$.

The following theorem ensures that Algorithm~\ref{alg:dt-learner} terminates and produces the correct output when score function $\scorer$ and the stopping criteria $\stopCriteria$ defined above is used as parameters.

\begin{theorem}\label{thm:DT-algo-correctness} 
	Given sample $\Traces$ and threshold $\thres\in\real$, Algorithm~\ref{alg:dt-learner} terminates and returns a decision tree over \LTLf{} formula $\DT$ such that $\loss(\Traces, \DT)\leq \thres$.
\end{theorem}

\begin{proof}
	First, observe that at each decision node, we can always infer an \LTL{} formula $\formula$ for which $\scorer(\Traces, \formula)\geq\minscore$, for any value of $\minscore$.
	This is because there always exists an LTL formula $\formula$ that produces perfect classification, and for this, $\scorer(\Traces,\formula)=1$.
	Second, observe that whenever a split is made during the learning algorithm, sub-samples $\Traces_1$ and $\Traces_2$ are both non-empty due to Lemma~\ref{lem:scorearl-guarentee}.
	This implies that the algorithm terminates since, a sample can be only split finitely many times.
	Now, for ensuring the decision tree $\DT$ achieves a $\loss(\Traces, \DT)\leq \thres$, we use induction over the structure of the decision tree. 
	If $\DT$ is leaf node $\ltrue$ or $\lfalse$, then $\loss(\Traces, \DT)\leq \thres$ using the stopping criteria.
	Now, say that $\DT$ is a decision tree with root $\formula$ and subtrees $\DT_1$ and $\DT_2$, meaning $\formula_\DT = (\formula\land\formula_{\DT_1})\lor(\neg\formula\land\formula_{\DT_2})$.
	Also, say that the sub-samples produced by $\formula$ are $\Traces_1$ and $\Traces_2$.
	By induction hypothesis, we can say that $\loss(\Traces_1, \DT_1)\leq\thres$ and $\loss(\Traces_2, \DT_2)\leq\thres$.
	Now, it is easy to observe that $\loss(\Traces_1, (\formula\land\formula_{\DT_1}))\leq\thres$ and $\loss(\Traces_2, (\neg\formula\land\formula_{\DT_2}))\leq\thres$, since $\formula$ satisfies all traces in $\Traces_1$ and $\neg\formula$ does not satisfy any trace in $\Traces_2$.
	We, thus, have $\loss(\Traces,\DT)=\loss(\Traces_1\uplus\Traces_2,(\formula\land\formula_{\DT_1})\lor(\neg\formula\land\formula_{\DT_2}))\leq\thres$
\end{proof}

\section{Experimental Evaluation}
\label{sec:evaluation}

\newcommand{\MaxSAT}{\emph{MaxSAT-flie}}
\newcommand{\longMaxSAT}{MaxSAT-based algorithm (Algorithm \ref{alg:max-sat-learner})}

\newcommand{\MaxSATDT}{\emph{MaxSAT-DT}}
\newcommand{\longMaxSATDT}{decision tree learning algorithm (Algorithm \ref{alg:dt-learner})}

\newcommand{\SAT}{\emph{SAT-flie}}
\newcommand{\longSAT}{SAT-based learning algorithms introduced by Neider and Gavran (Algorithm 1 from~\cite{DBLP:conf/fmcad/NeiderG18})}

\newcommand{\SATDT}{\emph{SAT-DT}}
\newcommand{\longSATDT}{decision tree based learning algorithm introduced by Neider and Gavran (Algorithm 2 from~\cite{DBLP:conf/fmcad/NeiderG18})}

\colorlet{colorThres0}{red!90!black}
\colorlet{colorThres5}{blue}
\colorlet{colorThres10}{green!80!purple}
\colorlet{colorMin80}{black}
\colorlet{colorMin60}{black}

\begin{table}[h]
\centering
\caption{Common \LTL{} patterns used in practice \cite{10.1145/298595.298598}}
\label{tab:LTL-patterns}
\begin{tabular}{ccc}
\hline
Absence &
Existence &
Universality
\\ \hline
$\lglobally( \lnot p_0)$ & 
$\leventually( p_0 )$ & 
$\lglobally( p_0 )$ 
\\
$\leventually(p_1) \limplies ( \lnot p_0 \luntil p_1 )$ & 
$\lglobally( \lnot p_0) \lor \leventually( p_0 \land \leventually( p_1 ) )$ & 
$\leventually( p_1 ) \limplies ( p_0 \luntil p_1 )$ 
\\
$\lglobally( p_1 \limplies \lglobally( \lnot p_0 ) )$ & 
$\lglobally( p_0 \land ( \lnot p_1 \limplies ( \lnot p_1 \luntil ( p_2 \land \lnot p_1 ) ) ) )$ & 
$\lglobally( p_1 \limplies \lglobally( p_0 ) )$ 
\\ \hline
\multicolumn{3}{c}{
Disjunction of common patterns
}
\\ \hline
\multicolumn{3}{c}{
$\lglobally( \lnot p_0) \lor \leventually( p_0 \land \leventually( p_1 ) ) \lor
 \lglobally( \lnot p_3) \lor \leventually( p_2 \land \leventually( p_3 ) )$ 
}
\\
\multicolumn{3}{c}{
$\leventually( p_2 ) \lor \leventually( p_0 ) \lor \leventually( p_1 )$ 
}
\\
\multicolumn{3}{c}{
$\lglobally( p_0 \land ( \lnot p_1 \limplies ( \lnot p_1 \luntil ( p_2 \land \lnot p_1 ) ) ) ) \lor
 \lglobally( p_3 \land ( \lnot p_4 \limplies ( \lnot p_4 \luntil ( p_5 \land \lnot p_4 ) ) ) )$ 
}
\\ \hline
\end{tabular}
\end{table}

In this section, we aim to evaluate the performance of our proposed algorithms
and compare them to the SAT-based learning algorithms by Neider and Gavran~\cite{DBLP:conf/fmcad/NeiderG18}.
We compare the following four algorithms:

\SAT{}: the \longSAT{},
\MaxSAT{}: our \longMaxSAT{},
\SATDT{}: the \longSATDT{}\footnote{We adapted \SATDT{} algorithm to have a stopping criterion similar to the one used in Section~\ref{subsec:stop-criteria}.} and
\MaxSATDT{}: our \longMaxSATDT{}.

We implement all learning algorithms in a Python tool\footnote{\url{https://github.com/cryhot/samples2LTL}}
using Microsoft Z3 \cite{DBLP:conf/tacas/MouraB08}.
All experiments were conducted on a Debian machine with Intel Xeon E7-8857 CPU at 3GHz using upto 6GB of RAM.

We generate samples based on common \LTL{} patterns \cite{10.1145/298595.298598} that we adapted to \LTLf{},
presented in Table \ref{tab:LTL-patterns}.
In a first benchmark (without noise), we generate 148 samples
with the generation method proposed by Neider and Gavran~\cite{DBLP:conf/fmcad/NeiderG18}.
The size of the generated samples ranges between 12 and 1000, consisting of traces of length up to 15.
Furthermore, we derive a second benchmark from the first one, by introducing 5\% noise:
for each sample of the benchmark,
we invert the labels of up to $5\%$ of the traces, randomly.

\begin{table}[t]
\caption{Summary of all the tested algorithms}
\label{tab:algo1-summary}
\resizebox{\linewidth}{!}{%
\begin{tabular}{|c|ccc|ccc|}
\hline
Algorithm &
\multicolumn{3}{c|}{benchmark without noise} &
\multicolumn{3}{c|}{benchmark with 5\% noise}
\\
&
\makecell{Number of \\ timeouts} &
\makecell{Avg. running \\ time in $s$} &
\makecell{Avg. inferred \\ formula size} &
\makecell{Number of \\ timeouts} &
\makecell{Avg. running \\ time in $s$} &
\makecell{Avg. inferred \\ formula size}
\\ \hline
\SAT{} &
 36/148 & 293.31 &  3.76 &
124/148 & 780.51 &  5.96
\\ \hline
\MaxSAT{}($\thres=0.001$) &
 47/148 & 357.26 &  3.47 &
130/148 & 801.03 &  4.89
\\
\MaxSAT{}($\thres=0.05$) &
 27/148 & 218.46 &  2.86 &
 87/148 & 548.65 &  2.95
\\
\MaxSAT{}($\thres=0.1$) &
 26/148 & 211.81 &  2.59 &
 40/148 & 275.97 &  2.54
\\ \hline
\SATDT{}($\thres=0.05$)&
 51/148 & 342.35 &  5.92 &
127/148 & 786.16 &  9.62
\\ \hline
\MaxSATDT{}($\thres=0.05, \minscore=0.8$) &
 23/148 & 174.58 &  6.77 &
 85/148 & 543.50 &  7.05
\\
\MaxSATDT{}($\thres=0.05, \minscore=0.6$) &
  7/148 &  74.97 & 30.91 &
 38/148 & 281.60 & 56.55
\\ \hline
\end{tabular}}
\end{table}

We evaluate the performance of all the algorithms
on the two benchmarks previously defined.
We set a timeout of $900 s$ on each run.
Table \ref{tab:algo1-summary} presents the parameters of the algorithms,
as well as their respective performances.

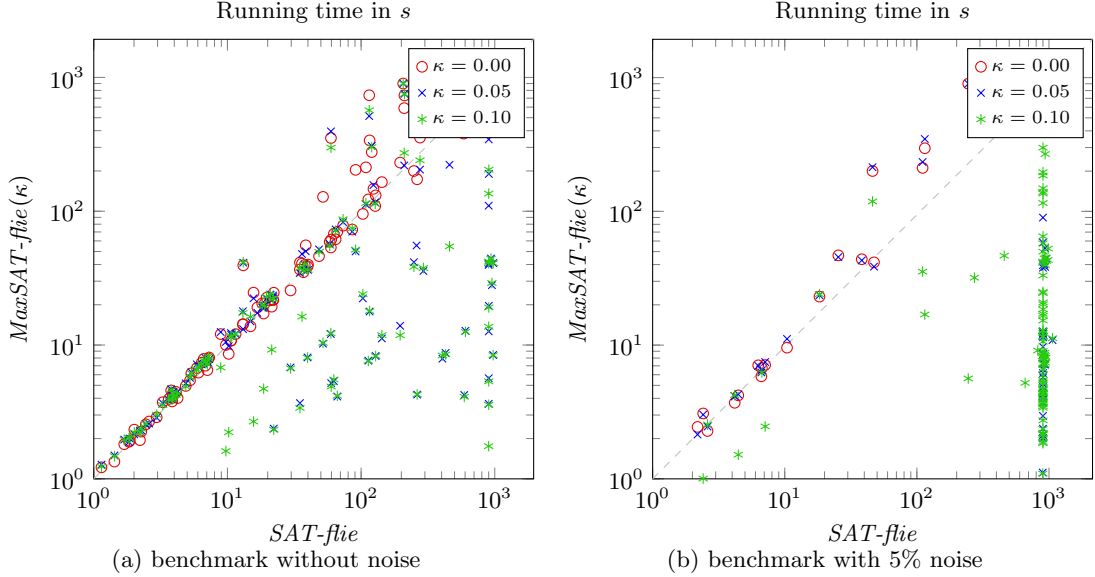
\begin{figure*}[t]
\centering
\begin{subfigure}[t]{.45\linewidth}
\centering

            \begin{tikzpicture}
            \begin{loglogaxis}
            [width = \linewidth,
height = \linewidth,
xmin = 1e0,
ymin = 1e0,
extra x ticks = {5000},
extra y ticks = {5000},
extra x tick labels = {\strut s},
extra y tick labels = {\strut s},
title = {Running time in $s$},
title style = {font=\footnotesize, inner sep=0pt},
xlabel = {\SAT {}},
x label style = {yshift=1mm},
xlabel near ticks,
ylabel = {\MaxSAT {}($\thres $)},
ylabel near ticks,
legend entries = {$\thres =0.00$,$\thres =0.05$,$\thres =0.10$},
legend style = {font=\scriptsize}]

                \addplot[only marks,
mark = o,
color = {colorThres0},
x filter/.code = {\edef\tempa{\thisrow{benchmark_noise}}\edef\tempb{0.0}\ifx\tempa\tempb\else\def\pgfmathresult{}\fi}
                ] table [x = {runtime_SAT},
y = {runtime_MaxSAT0},
col sep = comma
                ] {joinedByTraces.csv};

                \addplot[only marks,
mark = x,
color = {colorThres5},
x filter/.code = {\edef\tempa{\thisrow{benchmark_noise}}\edef\tempb{0.0}\ifx\tempa\tempb\else\def\pgfmathresult{}\fi}
                ] table [x = {runtime_SAT},
y = {runtime_MaxSAT5},
col sep = comma
                ] {joinedByTraces.csv};

                \addplot[only marks,
mark = asterisk,
color = {colorThres10},
x filter/.code = {\edef\tempa{\thisrow{benchmark_noise}}\edef\tempb{0.0}\ifx\tempa\tempb\else\def\pgfmathresult{}\fi}
                ] table [x = {runtime_SAT},
y = {runtime_MaxSAT10},
col sep = comma
                ] {joinedByTraces.csv};

            \draw[black!25,dashed] (rel axis cs:0, 0) -- (rel axis cs:1, 1);

            \end{loglogaxis}
            \end{tikzpicture}

        
\vspace{-3mm}
\caption{benchmark without noise}
\label{fig:sat-vs-maxsat:noise0:time}
\end{subfigure}%
\begin{subfigure}[t]{.45\linewidth}
\centering

            \begin{tikzpicture}
            \begin{loglogaxis}
            [width = \linewidth,
height = \linewidth,
xmin = 1e0,
ymin = 1e0,
extra x ticks = {5000},
extra y ticks = {5000},
extra x tick labels = {\strut s},
extra y tick labels = {\strut s},
title = {Running time in $s$},
title style = {font=\footnotesize, inner sep=0pt},
xlabel = {\SAT {}},
x label style = {yshift=1mm},
xlabel near ticks,
ylabel = {\MaxSAT {}($\thres $)},
ylabel near ticks,
legend entries = {$\thres =0.00$,$\thres =0.05$,$\thres =0.10$},
legend style = {font=\scriptsize}]

                \addplot[only marks,
mark = o,
color = {colorThres0},
x filter/.code = {\edef\tempa{\thisrow{benchmark_noise}}\edef\tempb{0.05}\ifx\tempa\tempb\else\def\pgfmathresult{}\fi}
                ] table [x = {runtime_SAT},
y = {runtime_MaxSAT0},
col sep = comma
                ] {joinedByTraces.csv};

                \addplot[only marks,
mark = x,
color = {colorThres5},
x filter/.code = {\edef\tempa{\thisrow{benchmark_noise}}\edef\tempb{0.05}\ifx\tempa\tempb\else\def\pgfmathresult{}\fi}
                ] table [x = {runtime_SAT},
y = {runtime_MaxSAT5},
col sep = comma
                ] {joinedByTraces.csv};

                \addplot[only marks,
mark = asterisk,
color = {colorThres10},
x filter/.code = {\edef\tempa{\thisrow{benchmark_noise}}\edef\tempb{0.05}\ifx\tempa\tempb\else\def\pgfmathresult{}\fi}
                ] table [x = {runtime_SAT},
y = {runtime_MaxSAT10},
col sep = comma
                ] {joinedByTraces.csv};

            \draw[black!25,dashed] (rel axis cs:0, 0) -- (rel axis cs:1, 1);

            \end{loglogaxis}
            \end{tikzpicture}

        
\vspace{-3mm}
\caption{benchmark with 5\% noise}
\label{fig:sat-vs-maxsat:noise5:time}
\end{subfigure}
\vspace{-2mm}
\caption{Running time comparison of \SAT{} and \MaxSAT{}}
\label{fig:sat-vs-maxsat:time}
\end{figure*}

\begin{figure}[t]
\centering
\begin{subfigure}{.5\linewidth}
\centering
\begin{tikzpicture}
\begin{axis}
[
	width=\linewidth, height=.5\linewidth,
	title = {benchmark without noise},
	xlabel = {Running time ratio},
	xmin=-3.5, xmax=1.5,
	xtick={-3,...,1}, xticklabels={$10^{-3}$,$10^{-2}$,$10^{-1}$,$10^{0}$,$10^{1}$},
	ytick={1,2,3}, y dir=reverse, yticklabels={$\thres=0.00$,$\thres=0.05$,$\thres=0.10$},
	title style = {font=\footnotesize, inner sep=0pt},
	x label style = {yshift=1mm},
]

\addplot+[
	boxplot,
	mark=o,
	color={colorThres0},
] table [
	y expr = ln( \thisrow{Time on perfect (maxSAT0)} / \thisrow{Time on perfect (SAT)} )/ln(10),
] {SATvsMaxSAT.csv};

\addplot+[
	boxplot,
	mark=o,
	color={colorThres5},
] table [
	y expr = ln( \thisrow{Time on perfect (maxSAT5)} / \thisrow{Time on perfect (SAT)} )/ln(10),
] {SATvsMaxSAT.csv}
;

\addplot+[
	boxplot,
	mark=o,
	color={colorThres10},
] table [
	y expr = ln( \thisrow{Time on perfect (maxSAT10)} / \thisrow{Time on perfect (SAT)} )/ln(10),
] {SATvsMaxSAT.csv}
;

\end{axis}
\end{tikzpicture}
\label{fig:sat-vs-maxsat:time-boxplot:noise0}
\end{subfigure}%
\begin{subfigure}{.5\linewidth}
\centering
\begin{tikzpicture}
\begin{axis}
[
	width=\linewidth, height=.5\linewidth,
	title = {benchmark with 5\% noise},,
	xlabel = {Running time ratio},
	xmin=-3.5, xmax=1.5,
	xtick={-3,...,1}, xticklabels={$10^{-3}$,$10^{-2}$,$10^{-1}$,$10^{0}$,$10^{1}$},
	ytick={1,2,3}, y dir=reverse, yticklabels={$\thres=0.00$,$\thres=0.05$,$\thres=0.10$},
	title style = {font=\footnotesize, inner sep=0pt},
	x label style = {yshift=1mm},
]

\addplot+[
	boxplot,
	mark=o,
	color={colorThres0},
] table [
	y expr = ln( \thisrow{Time (maxSAT0)} / \thisrow{Time (SAT)} )/ln(10),
] {SATvsMaxSAT.csv}
;

\addplot+[
	boxplot,
	mark=o,
	color={colorThres5},
] table [
	y expr = ln( \thisrow{Time (maxSAT5)} / \thisrow{Time (SAT)} )/ln(10),
] {SATvsMaxSAT.csv}
;

\addplot+[
	boxplot,
	mark=o,
	color={colorThres10},
] table [
	y expr = ln( \thisrow{Time (maxSAT10)} / \thisrow{Time (SAT)} )/ln(10),
] {SATvsMaxSAT.csv}
;

\end{axis}
\end{tikzpicture}
\label{fig:sat-vs-maxsat:time-boxplot:noise5}
\end{subfigure}
\vspace{-3mm}
\caption{
	Comparison of the ratio of the running time of \MaxSAT{}($\thres$) over the running time of \SAT{} for all samples in the benchmarks.
}
\label{fig:sat-vs-maxsat:time-boxplot}
\end{figure}
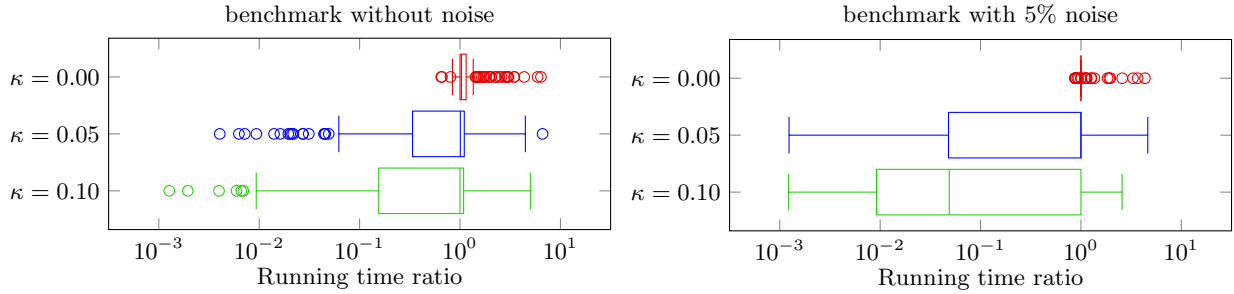

\begin{figure*}[t]
\centering

\begin{subfigure}{.37\linewidth}
\centering

            \begin{tikzpicture}
            \begin{axis}
            [
                width = \linewidth,
height = \linewidth,
xmin = 1,
ymin = 1,
xmax = 8,
ymax = 8,
xtick = {1,...,8},
ytick = {1,...,8},
xticklabels = {1,2,3,4,5,6,7,$\emptyset$},
yticklabels = {1,2,3,4,5,6,7,$\emptyset$},
enlarge x limits,
enlarge y limits,
title = {Inferred \LTLf {} formula size},
title style = {font=\footnotesize, inner sep=0pt},
xlabel = {\SAT {}},
xlabel near ticks,
ylabel = {\MaxSAT {}($\thres =0.10$)},
ylabel near ticks
            ]

            \addplot[scatter=true,
            	only marks,
            	visualization depends on = {sqrt(\thisrow{Freq-perf3}/148)*0.085*\linewidth \as \perpointmarksize},
            	point meta={TeX code symbolic={
            		\ifthenelse{\( \thisrow{SAT-perf3}>7 \OR \thisrow{MaxSAT10-perf}>7 \)}
            		{\edef\pgfplotspointmeta{timeout}}
            		{\edef\pgfplotspointmeta{normal}}
            	}},
            	scatter/classes={
            		normal={fill=gray,/tikz/mark size=\perpointmarksize},
            		timeout={fill=red,/tikz/mark size=\perpointmarksize}
            	},
            ] table [
            	x={SAT-perf3},
            	y={MaxSAT10-perf},
                col sep=comma,
            ] {SizeFreq-MaxSATvsSAT-timeouts.csv};

            \draw[black!25,dashed] (rel axis cs:0, 0) -- (rel axis cs:1, 1);

            \end{axis}
            \end{tikzpicture}

        
\vspace{-3mm}
\caption{benchmark without noise}
\label{fig:sat-vs-maxsat:noise0:size}
\end{subfigure}%
\begin{subfigure}{.37\linewidth}
\centering

            \begin{tikzpicture}
            \begin{axis}
            [
                width = \linewidth,
height = \linewidth,
xmin = 1,
ymin = 1,
xmax = 8,
ymax = 8,
xtick = {1,...,8},
ytick = {1,...,8},
xticklabels = {1,2,3,4,5,6,7,$\emptyset$},
yticklabels = {1,2,3,4,5,6,7,$\emptyset$},
enlarge x limits,
enlarge y limits,
title = {Inferred \LTLf {} formula size},
title style = {font=\footnotesize, inner sep=0pt},
xlabel = {\SAT {}},
xlabel near ticks,
ylabel = {\MaxSAT {}($\thres =0.10$)},
ylabel near ticks
            ]

            \addplot[scatter=true,
            	only marks,
            	visualization depends on = {sqrt(\thisrow{Freq-noisy3}/148)*0.085*\linewidth \as \perpointmarksize},
            	point meta={TeX code symbolic={
            		\ifthenelse{\( \thisrow{SAT-noisy3}>7 \OR \thisrow{MaxSAT10-noisy}>7 \)}
            		{\edef\pgfplotspointmeta{timeout}}
            		{\edef\pgfplotspointmeta{normal}}
            	}},
            	scatter/classes={
            		normal={fill=gray,/tikz/mark size=\perpointmarksize},
            		timeout={fill=red,/tikz/mark size=\perpointmarksize}
            	},
            ] table [
            	x={SAT-noisy3},
            	y={MaxSAT10-noisy},
                col sep=comma,
            ] {SizeFreq-MaxSATvsSAT-timeouts.csv};

            \draw[black!25,dashed] (rel axis cs:0, 0) -- (rel axis cs:1, 1);

            \end{axis}
            \end{tikzpicture}

        
\vspace{-3mm}
\caption{benchmark with 5\% noise}
\label{fig:sat-vs-maxsat:noise5:size}
\end{subfigure}
\vspace{-2mm}
\caption{
	Inferred \LTLf{} formula size comparison of \SAT{} and \MaxSAT{} with threshold $\thres=0.10$ on all samples.
	The surface of a bubble is proportional to the number of samples it represents.
	The timed out instances are represented by $\emptyset$.
}
\label{fig:sat-vs-maxsat:size}
\end{figure*}
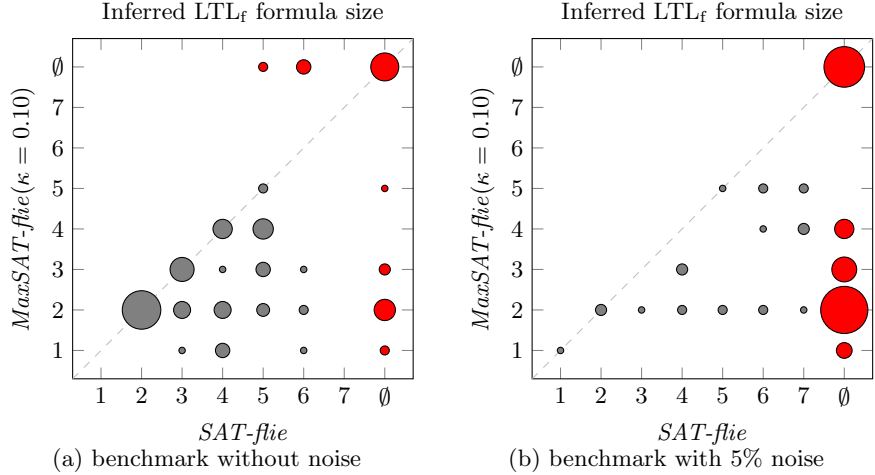

We first compare \MaxSAT{} (proposed in this paper) and \SAT{} (proposed in \cite{DBLP:conf/fmcad/NeiderG18}).
Figure \ref{fig:sat-vs-maxsat:time} presents a comparison of the running time of these two algorithms,
on each sample of the benchmark.
With $\thres=0.001$, \MaxSAT{} performs worse than \SAT{}.
This is largely due to the fact that SAT solvers are specifically designed to handle this type of problem.
For greater values of $\thres$, \MaxSAT{} performs better than \SAT{},
especially on the benchmark with noise (Figure \ref{fig:sat-vs-maxsat:noise5:time}).
To affirm this claim,
we calculate the ratio of the running times of \MaxSAT{} and \SAT{}
for each sample of the benchmarks (Figure \ref{fig:sat-vs-maxsat:time-boxplot}).
For example, given a sample $\Traces$,
this ratio would be the running time of \MaxSAT{} on $\Traces$
divided by the running time of \SAT{} on $\Traces$.

We evaluate the size of the inferred \LTLf{} formula by \MaxSAT{} and \SAT{}
on each sample of the benchmark in Figure \ref{fig:sat-vs-maxsat:size}.
The size of the formula inferred by \MaxSAT{} will by design be less than or equal to
the size of the formula inferred by \SAT{}.
As the running time of both algorithms grows exponentially with the number of iterations,
it is lower for \MaxSAT{} when the inferred formula size is strictly lower than the size of the formula inferred by \SAT{}.
However, when both inferred formulas have the same size,
there is no running time gain,
hence the median running time often being equal to 1 in Figure \ref{fig:sat-vs-maxsat:time-boxplot}.

\begin{figure}[t]
\centering
\begin{subfigure}{.5\linewidth}
\centering
\begin{tikzpicture}
\begin{axis}
[
	width=\linewidth, height=.45\linewidth,
	title = {All benchmarks},
	xlabel = {Running time ratio},
	ylabel = {$\minscore$},
	xmin=-4.25, xmax=2.75,
	xtick={-4,...,2}, xticklabels={$10^{-4}$,$10^{-3}$,$10^{-2}$,$10^{-1}$,$10^{0}$,$10^{1}$,$10^{2}$},
	ytick={1,2}, y dir=reverse, yticklabels={$0.8$,$0.6$},
	title style = {font=\footnotesize, inner sep=0pt},
	x label style = {yshift=1mm},
]

\addplot+[
	boxplot,
	mark=o,
	color={colorMin80},
] table [
	y expr = ln( \thisrow{runtime_MaxSATDT80} / \thisrow{runtime_MaxSAT5} )/ln(10),
] {joinedByTraces-sucDT80.csv}
;

\addplot+[
	boxplot,
	mark=o,
	color={colorMin60},
] table [
	y expr = ln( \thisrow{runtime_MaxSATDT60} / \thisrow{runtime_MaxSAT5} )/ln(10),
] {joinedByTraces-sucDT60.csv}
;

\end{axis}
\end{tikzpicture}
\label{fig:maxsat-vs-maxsat-dt:time-boxplot:noise0}
\end{subfigure}%
\begin{subfigure}{.5\linewidth}
\centering
\begin{tikzpicture}
\begin{axis}
[
	width=\linewidth, height=.45\linewidth,
	title = {All benchmarks},
	xlabel = {Inferred \LTLf{} formula size ratio},
	ylabel = {$\minscore$},
	xmin=-0.75, xmax=2.25,
	xtick={0,...,2}, xticklabels={$10^{0}$,$10^{1}$,$10^{2}$},
	ytick={1,2}, y dir=reverse, yticklabels={$0.8$,$0.6$},
	title style = {font=\footnotesize, inner sep=0pt},
	x label style = {yshift=1mm},
]

\addplot+[
	boxplot,
	mark=o,
	color={colorMin80},
] table [
	y expr = ln( \thisrow{LTL_size_MaxSATDT80} / \thisrow{LTL_size_MaxSAT5} )/ln(10),
] {joinedByTraces-sucDT80.csv}
;

\addplot+[
	boxplot,
	mark=o,
	color={colorMin60},
] table [
	y expr = ln( \thisrow{LTL_size_MaxSATDT60} / \thisrow{LTL_size_MaxSAT5} )/ln(10),
] {joinedByTraces-sucDT60.csv}
;

\end{axis}
\end{tikzpicture}
\label{fig:maxsat-vs-maxsat-dt:time-boxplot:noise5}
\end{subfigure}
\vspace{-3mm}
\caption{
	On each sample of the benchmarks, comparison of the ratio of the performances of \MaxSATDT{}($\minscore$) over the performances of \MaxSAT{}, with $\thres=0.05$ for both algorithms, and where both algorithms did not time out.
}
\label{fig:maxsat-vs-maxsat-dt:time-boxplot}
\end{figure}
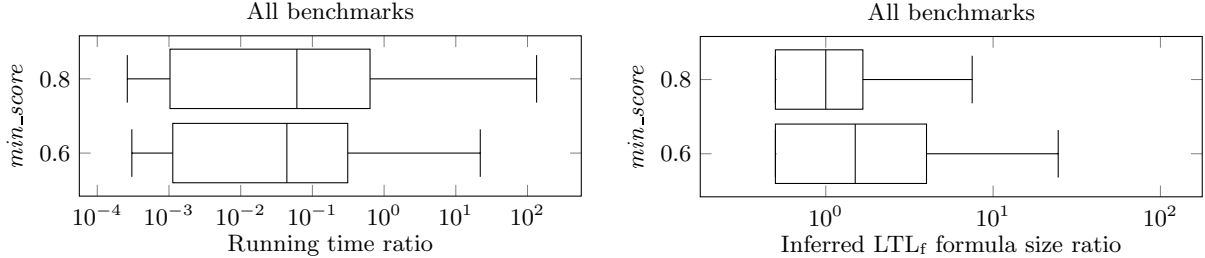

We now compare the two algorithms proposed in this paper: did \MaxSATDT{} perform any better than \MaxSAT{}?
To be able to compare learned decision trees to learned \LTLf{} formulas,
we measure the size of a tree $t$ in terms of the size of the formula $\formula_t$ this tree encodes.
Figure \ref{fig:maxsat-vs-maxsat-dt:time-boxplot} presents a comparison of the running time ratio as well as the inferred formula size ratio of these two algorithms,
on each sample of the benchmark that did not time out with both algorithms.
We observe that the running time is generally lower for \MaxSATDT{} than for \MaxSAT{}.
However, \MaxSATDT{} tends to infer larger formulas than formulas inferred by \MaxSAT{}.
This trade-off between running time and inferred formula size is more pronounced for lower values of $\minscore$.

Regarding \SATDT{} (proposed in \cite{DBLP:conf/fmcad/NeiderG18}), we observe a large number of timeouts,
especially when evaluated on the benchmark with 5\% noise.

\section{Conclusion}

We have developed two novel algorithms for inferring \LTLf{} formulas from a set of labeled traces allowing misclassifications. We have demonstrated that our algorithms are efficient in inferring formulas, especially from noisy data. As a part of future work, we like to apply our MaxSAT-based approach for inferring models in other formalisms that incorporate SAT-based learning (e.g.~\cite{DBLP:conf/ijcai/0002FN20}).

\section{Acknowledgements}

This material is based upon work supported by the Defense Advanced Research Projects Agency (DARPA) under Contract No. HR001120C0032, ARL W911NF2020132, ARL ACC-APG-RTP W911NF, NSF 1646522 and DFG grant no. 434592664.

\bibliographystyle{splncs04}
\bibliography{bib}

\end{document}